%% file: main.tex
\documentclass{article}
\input{pretext}

\title{Memory Limitations of Prompt Tuning in Transformers}

\author{%
Maxime Meyer$^{1,2}$ \quad Mario Michelessa$^{2,3}$ \quad Caroline Chaux$^2$ \quad Vincent Y. F. Tan$^{1,4}$\\
$^1$Department of Mathematics, National University of Singapore, Singapore, 117543 \\ $^2$IPAL, IRL2955, Singapore\\$^3$School of Computing, National University of Singapore, Singapore, 117543\\$^4$Department of Electrical and Computer Engineering, National University of Singapore, Singapore, 117543\\
\texttt{\{maxime.meyer,mario.michelessa,vtan\}@u.nus.edu}\\
\texttt{caroline.chaux@cnrs.fr}
}

\begin{document}

\maketitle

\begin{abstract}
Despite the empirical success of prompt tuning in adapting pretrained language models to new tasks, theoretical analyses of its capabilities remain limited. Existing theoretical work primarily addresses universal approximation properties, demonstrating results comparable to standard weight tuning. In this paper, we explore a different aspect of the theory of transformers---the memorization capability of prompt tuning---and provide two principal theoretical contributions. First, we prove that the amount of information memorized by a transformer cannot scale faster than linearly with the prompt length. Second, and more importantly, we present the first formal proof of a phenomenon empirically observed in large language models: performance degradation in transformers with extended contexts. We rigorously demonstrate that transformers inherently have limited memory, constraining the amount of information they can retain, regardless of the context size. This finding offers a fundamental understanding of the intrinsic limitations of transformer architectures, particularly their ability to handle long sequences.
\end{abstract}

\input{ch_introduction}

\input{ch_related-works}

\input{ch_preliminaries}

\input{ch_covering-number}

\input{ch_limited-memory}

\input{ch_single-layer}

\section{Conclusion}

This work advances the theoretical understanding of prompt tuning by analyzing the memorization capacity of transformer architectures. We established that the amount of information a transformer can reliably encode through prompt tuning is fundamentally limited as the number of input/output couples that can be memorized scales at most linearly with the prompt length. More importantly, we provided the first formal justification for the empirically observed degradation in transformer performance with long contexts. Our results demonstrate that transformers suffer from an intrinsic memory limitation, independent of context length. These findings highlight a fundamental limitation in the use of prompt tuning in transformer-based models.

\bibliographystyle{plainnat}

\input{main.bbl}
\appendix

\input{ch_appendices}

\end{document}

%% file: pretext.tex
\usepackage[authoryear]{natbib} 
\usepackage{amsmath, amssymb, booktabs}
\usepackage{changepage}
\usepackage{array}
\usepackage{microtype}
\usepackage{graphicx}
\usepackage{subfigure}
\usepackage{booktabs} 
\usepackage[colorlinks=true, linkcolor=blue, citecolor=blue, urlcolor=blue]{hyperref}
\usepackage[capitalize,noabbrev]{cleveref} 

\usepackage[utf8]{inputenc} 
\usepackage[T1]{fontenc}    
\usepackage{hyperref}       
\usepackage{url}            
\usepackage{booktabs}       
\usepackage{amsfonts}       
\usepackage{nicefrac}       
\usepackage{microtype}      
\usepackage{xcolor}         

\usepackage{amsmath}
\usepackage{amssymb}
\usepackage{mathtools}
\usepackage{amsthm}
\usepackage[capitalize,noabbrev]{cleveref}
\usepackage{thmtools}
\theoremstyle{plain}
\newtheorem{theorem}{Theorem}[section]
\newtheorem{proposition}[theorem]{Proposition}
\newtheorem{lemma}[theorem]{Lemma}

\newtheorem{assumption}[theorem]{Assumption}
\theoremstyle{definition}
\newtheorem{definition}[theorem]{Definition}
\theoremstyle{remark}
\newtheorem{remark}[theorem]{Remark}

\crefname{assumption}{Assumption}{Assumptions}
\Crefname{assumption}{Assumption}{Assumptions}
\usepackage[preprint]{neurips_2025}

%% file: ch_introduction.tex
\section{Introduction}

Since their introduction in~\citet{Vaswani_A_2017_p-nips-OG}, transformers have reshaped the landscape of machine learning, achieving state-of-the-art results in tasks ranging from natural language processing~\citep{wolf-etal-2020-transformers} to computer vision~\citep{dosovitskiy2021an,zhao2021point,zhai2022scaling}. A central factor in their success is the ability to adapt pretrained models to downstream tasks via fine-tuning~\citep{howard-ruder-2018-universal}. Traditionally, this process involves updating the full set of model parameters---a strategy that becomes increasingly costly as models and datasets scale~\citep{NEURIPS2020_1457c0d6,10.5555/3648699.3648939}.

For large language models (LLMs), prompt tuning has emerged as a parameter-efficient alternative. It enables task adaptation with minimal modification to the model architecture by prepending a small set of learnable parameters (commonly refered as pre-prompt) to the input (e.g. human-readable text for prompt engineering~\citep{chen2023plot,wen2023hard} or continuous representations for soft-token optimization~\citep{li2021prefix,lester-etal-2021-power}). Despite its simplicity, prompt tuning often matches or exceeds the performance of full fine-tuning, while updating only a small fraction of the model’s parameters~\citep{liu-etal-2022-p,10.5555/3600270.3602070,10.5555/3600270.3601883,wang2023multitask,10203359,gao2024protein,shi2024dept,fu2024nemesis}.

Prompt tuning proves especially effective in settings like in-context learning~\citep{dong-etal-2024-survey}, where the pre-prompt includes multiple input/output instances, e.g. ``the answer to $\mathbf X^1$ is $\mathbf Y^1$, the answer to $\mathbf X^2$ is $\mathbf Y^2$, \dots, the answer to $\mathbf X^k$ is $\mathbf Y^k$'' for distinct queries $\mathbf X^1, \dots, \mathbf X^k$. As models scale, one might expect that providing longer pre-prompts should allow a sufficiently large transformer to memorize and generalize from more such examples. However, recent empirical studies suggest that even advanced LLMs struggle to memorize information presented in long prompts~\citep{hsieh2024ruler,levy-etal-2024-task,laban-etal-2024-summary,li-etal-2024-loogle,liu-etal-2024-lost,jin2025longcontext}, despite architectural support~\citep{10.5555/3692070.3692512,li2024extending}. 
In our work, we formally demonstrate and characterize the limited memory not only of current LLMs, but of all transformers when using prompt tuning, by studying the following questions:

{\it Can the pre-prompt length be shortened while retaining information to improve the memorization ability of transformers? 
Or do transformers have inherent limitations as to the amount of information a pre-prompt can transfer?}

We provide the first formal proof of the limited memorization capability of transformers using prompt tuning. Specifically, we prove that there exists an upper bound on the amount of new information a transformer can learn through prompt-tuning, irrespective of pre-prompt length. 
We further prove that encoding the information in the pre-prompt as input/output pairs is optimal in the sense that the number of distinct items a transformer can memorize via prompt tuning cannot grow faster than linearly with prompt length. 

These results reveal fundamental limitations of prompt-based adaptation.
Notably, for in-context learning where the pre-prompt consists of a list of $k$ input/output pairs, one would expect a sufficiently large LLM to be able to memorize this data set---that is to output the correct $\mathbf Y^i$ when queried with $\mathbf X^i$---regardless of $k$ or the specific values. However, our findings show the opposite. 
In addition, they imply that the asymptotical gain of soft prompt optimization over prompt engineering is at most linear. This implication complements the main theorem from~\citep{petrov2024when}, which identified a case where soft prompt optimization offers a polynomial gain of order the embedding dimension $d$ over prompt engineering.

\subsection{Contributions}

Our contributions can be summarized as follows.
\begin{itemize}
    \item We characterize the maximal memorization capability of a transformer with respect to prompt tuning. That is we obtain an integer $K$---depending on the parameters of the transformer and on an upper bound $R$ of the magnitude of tokens---such that for any inputs $\mathbf X^1,...,\mathbf X^k\in\mathbb R^{d\times m}$, the proportion of outputs that are accessible through prompt tuning decreases exponentially with $k$ from the rank $k\ge K$ (Section~\ref{sec:transformer_memory}).
    \item We characterize the scaling between the pre-prompt size $m_p$ and the amount of new information learned by the transformer in Theorem~\ref{thm:prompt_information}. More formally, we show that the maximal number $k$ of input/output pairs of length $m$ a transformer can reliably learn through prompt tuning of length $m_p$ scales as $k\in O(\frac{m_p}m)$ (Section~\ref{sec:prompt_information}).
    \item We expand on the result from~\citet{Wang_Y_2023_p-nips_prompt-og} and show that, even with no assumptions, a one layer transformer has very little expressivity with regards to prompt tuning. Specifically, we show that the space accessible through prompt tuning for a pair of inputs $(\mathbf X^1,\mathbf X^2)$ is almost a hyperplane of the output space. We also generalize this result to approximate memorization, where the goal is to approximate the outputs up to an error $\varepsilon$ (Section~\ref{sec:single_layer}).
    \item It is important to note that our first two contributions (Sections~\ref{sec:prompt_information}~and~\ref{sec:transformer_memory}) also hold for masked self-attention. To the best of our knowledge, we are the first work on the approximation theory of prompt tuning to consider this architecture.
\end{itemize}

%% file: ch_related-works.tex
\subsection{Related Works}

\textbf{Universality of Prompt Tuning.} Despite its empirical success, the theoretical understanding of prompt tuning remains limited. Prior work has primarily analyzed its universal approximation properties:
\begin{itemize}
    \item~\citet{Wang_Y_2023_p-nips_prompt-og} show that prompt tuning is universal in the sense that for every Lipschitz constant $L>0$ and error $\varepsilon$, there exists a transformer $\tau$ that can approximate any $L$-Lipschitz function up to error $\varepsilon$ through prompt tuning.
    \item~\citet{Petrov_A_2024_workshop-iclr_prompt-false-architecture, Hu_J_2025_p-iclr_prompt-smaller} improve this result by further upper bounding the required size of the transformer $\tau$.
    \item~\citet{Nakada_R_2025_arxiv_prompt-smooth} achieve similar results for approximating smooth---rather than Lipschitz---functions.
    \end{itemize}
These studies construct synthetic transformers with weights designed to facilitate prompt tuning. The constructions they use demonstrate approximation power but fail to capture the behavior of pretrained models used in practice. In our work, we depart from this stylized setting by studying the prompt tuning capabilities of an \emph{arbitrary} transformer.

\textbf{Comparisons Between Full, Soft Prompt, and Hard Prompt Tuning.} 
The difference between the performances of prompt tuning and full weight fine-tuning was analysed in~\citet{10.5555/3618408.3619520,petrov2024when}. \citet{petrov2024when} also study the difference between soft prompt tuning and prompt engineering, identifying a setting in which soft prompt tuning is more expressive. Contrary to our work, their analysis is however based on the construction of a specific transformer, and might not hold for any weight attribution.

\textbf{Memorization Capability of Transformer.}
There are many works on the memorization capabilities of the transformer architecture. \citet{kim2023provable} prove that transformers with $2n$ self-attention layers
suffice for the memorization of length-$n$ inputs. This result was improved to single-layer transformers in~\citet{mahdavi2024memorization,kajitsuka2024are}. More recently,~\citet{kajitsuka2025on} study the optimal number of parameters required for memorization.

However, the memorization capability of prompt tuning has been very little studied. The limitations of the memorization capability of a single-layer, single-head transformer satisfying a few assumptions were shown in~\citet{Wang_Y_2023_p-nips_prompt-og}. We generalize this setting to any transformer, getting rid of all assumptions. \citet{Hu_J_2025_p-iclr_prompt-smaller} construct a specific transformer that can memorize datasets. However, this {\it ad hoc} transformer is far from the ones being used in practice. Indeed, even for the simple task of memorizing two input/output pairs of size $2$, their construction requires a pre-prompt size of length more then $10^{640}$ and more than $10^{10^{640}}$ neurons\footnote{Those numbers come directly from Theorem 2.5 of~\citet{Hu_J_2025_p-iclr_prompt-smaller}. We considered the $\ell_2$ norm, with the embedding dimension of BERT-base~\citep{Devlin_J_2019_j-acl_BERT} ($d=768$), assumed that the input/output pairs $(\mathbf X^i,\mathbf Y^i)_{i\in\{1,2\}}$ are generated by a $1$-Lipshitz function ($\|\mathbf Y^1-\mathbf Y^2\|_2\leq\|\mathbf X^1-\mathbf X^2\|_2$), and aimed for an approximation error of $\varepsilon=30$ (generous considering that in practice, this is the order of the maximal norm of an embedded vector in BERT-base~\citep{kobayashi-etal-2020-attention}).}---recall that there are ``only'' about $10^{80}$ atoms in the universe.

\subsection{Notations}

We emphasize that most of our results hold for an arbitrary norm. We denote this norm by $\|\cdot\|$, which can encompass any $\ell_p$ norm for $p\ge1$, as well as the $\ell_\infty$ norm for vectors, and any $\ell_p$ or entrywise $\ell_p$ norm for $p\ge1$, as well as the Frobenius norm $\|\cdot\|_{\mathrm F}$, max norm and the spectral norm $\|\cdot\|_2$ for matrices.

We adopt the notations of~\citep{Wang_Y_2023_p-nips_prompt-og}. Bold lowercase letters (e.g., $\mathbf{x}$) denote vectors, and bold uppercase letters (e.g., $\mathbf{W}$) denote matrices. For a matrix $\mathbf{W}$, we write $\mathbf{W}_{i,j}$, $\mathbf{W}_{i,:}$, and $\mathbf{W}_{:,j}$ to refer to its $(i,j)$-th element, $i$-th row, and $j$-th column, respectively. We can use $i=-1$ and $j=-1$ to denote the last row and column, respectively. And $\mathbf{W}_{:,j:}$ refers to the submatrix with the first $j$ columns removed. Superscripts are used to index a sequence of matrices: for example, $\mathbf X^i$ denotes the $i$-th matrix in a sequence. The concatenation of $k$ matrices of same row number $d$, $\mathbf W^1\in\mathbb R^{d\times m_i},\dots,\mathbf W^k\in\mathbb R^{d\times m_k}$ is denoted by $[\mathbf W^1,\dots,\mathbf W^k]\in\mathbb R^{d\times\sum_{i=1}^km_i}$. We write $B^n(0,r):=\{\mathbf X\in\mathbb R^{d\times m},\|\mathbf X\|<r\}$ the open ball centered in $0$ of radius $r$, and define the embedding radius $r$ as the maximal norm of an input/output vector. Recall that the distance between two sets $\mathcal A$ and $\mathcal B$ is defined as $d(\mathcal A,\mathcal B)=\inf_{a\in\mathcal A,b\in\mathcal B}\|a-b\|$, and that the volume of a set $\mathcal A\subset\mathbb R^n$ is defined as $\mathrm{Vol}(\mathcal A) := \int_{\mathbb{R}^n} \mathbf{1}_\mathcal A(x) \, \mathrm{d}x$, where \( \mathbf{1}_\mathcal A \) is the indicator function of \( \mathcal A \). The diameter of $\mathcal A$ is $\mathrm{Diam}(\mathcal A) := \sup_{x, y \in \mathcal A} d(x, y)$.

We write $\sigma$ to denote the softmax function. The rectified linear unit is defined as $\mathrm{ReLU}(\mathbf{v}) = \max(\mathbf{v}, \mathbf{0})$, where the maximum is applied elementwise. We denote \( \mathcal{P}_{\mathrm c}(\mathbb{R}^d) \) the set of compactly supported probability measures on \( \mathbb{R}^d \).

%% file: ch_preliminaries.tex
\section{Standard and Mean-Field Transformers}

\subsection{The Transformer Architecture}

We consider transformer networks~\citep{Vaswani_A_2017_p-nips-OG} composed of repeated self-attention and MLP layers. Let $\mathbf{X} \in \mathbb{R}^{d \times m}$ denote the input sequence of $m$ tokens, each of dimension $d$. A single self-attention layer is defined as follows.

\begin{definition}[Self-Attention Layer]\label{def:attention}
An $h$-head self-attention layer maps every query token $\mathbf{x}$ within a context $\mathbf{X}\in\mathbb R^{d\times m}$ to
\begin{equation*}
\operatorname{Att}(\mathbf{x}, \mathbf{X}) = \sum_{i=1}^{h} \mathbf{W}_{\mathrm o}^i \mathbf{W}_{\mathrm v}^i \mathbf{X} \cdot \sigma\big((\mathbf{W}_{\mathrm k}^i \mathbf{X})^\top \mathbf{W}_{\mathrm q}^i \mathbf{x} \big),
\label{eq:attention}
\end{equation*}
where $\mathbf{W}^i_{\mathrm q},\mathbf{W}^i_{\mathrm k}\in\mathbb R^{s\times d},\mathbf{W}^i_{\mathrm v}\in\mathbb R^{s'\times d},\mathbf{W}^i_{\mathrm o}\in\mathbb R^{d\times s'}$, $s$ is called the hidden dimension and is typically taken to be $s=s'=\frac d h$~\citep{Vaswani_A_2017_p-nips-OG,NEURIPS2020_1457c0d6,dosovitskiy2021an}, and the normalization factor $1/\sqrt{s}$ is absorbed into $\mathbf{W}_{\mathrm k}^i$ for notational simplicity. 
The results for each token are then concatenated to produce the output of the self-attention layer
\begin{align*}
f(\mathbf X)\colon&=\operatorname{Att}(\mathbf{X}, \mathbf{X}) \\&= 
[
\operatorname{Att}(\mathbf{X}_{:,1}, \mathbf{X}),\;
\ldots,\;
\operatorname{Att}(\mathbf{X}_{:,m}, \mathbf{X})
].
\label{eq:cross-attention}
\end{align*}
\end{definition}

We now define transformer networks, which are a stack of multiple transformer layers composed sequentially.

\begin{definition}[Transformer Layer]
    A transformer layer $$\mathcal L:\bigcup_{m=1}^{+\infty}\mathbb R^{d\times m}\longrightarrow\bigcup_{m=1}^{+\infty}\mathbb R^{d\times m},$$ is defined as
    \[\mathcal L(\mathbf{X}) = \operatorname{MLP}\big( \operatorname{Att}(\mathbf{X}, \mathbf{X}) + \mathbf{X} \big),\]
    with
    \begin{align*}
        \operatorname{MLP}(\mathbf{X}) &= \left[\mathbf{W}_2 \mathrm{ReLU}(\mathbf{W}_1 \mathbf{X}_{: , 1} + \mathbf{b}_1) + \mathbf{b}_2 + \mathbf{X}_{: , 1}, \right. \notag \\
        &\quad \left. \ldots, 
        \mathbf{W}_2 \mathrm{ReLU}(\mathbf{W}_1 \mathbf{X}_{: , m} + \mathbf{b}_1) + \mathbf{b}_2+\mathbf{X}_{: , m}\right].
    \end{align*}
    Here, $\mathbf W_1\in\mathbb R^{d_{\mathrm{ff}}\times d},\mathbf W_2\in\mathbb R^{d\times d_{\mathrm{ff}}},\mathbf b_1\in\mathbb R^{d_{\mathrm{ff}}},\mathbf b_2\in\mathbb R^d$, and $d_{\mathrm{ff}}$ is the hidden dimension of the MLP.
\end{definition}

For simplicity, the layer normalization operation is omitted, following~\citep{Kim_H_2021_p-icml_not-lipshitz}.

\begin{definition}[Transformer]
    An $l$-layer transformer of embedding dimension $d\in\mathbb N$ $$\tau:\bigcup_{m=1}^{+\infty}\mathbb R^{d\times m}\longrightarrow\bigcup_{m=1}^{+\infty}\mathbb R^{d\times m},$$ is defined as the composition of $l$ transformer layers $$\tau=\mathcal L_1\circ\dots\circ\mathcal L_l.$$
\end{definition}

\begin{remark}[On Positional Encodings]
Our definition of the transformer architecture omits positional encodings for clarity of presentation. However, since positional information is typically incorporated via the addition of a fixed or learned matrix to the input sequence, it can be absorbed into the input without affecting the structure of the analysis. As such, all of our results extend immediately to settings where positional encodings are present.
\end{remark}

\subsection{Mean-Field Transformers}\label{sec:mean-field}

The proof of Theorem~\ref{thm:mean-field} requires the mean-field generalization of the transformer architecture, as presented in~\citet{Castin_V_2024_p-icml_lipschitz}.

When the size of the pre-prompt is large, it can be convenient to interpret a transformer as a map between probability measures rather than finite-length sequences~\citep{pmlr-v151-sander22a,NEURIPS2023_b2b3e1d9}. This viewpoint is motivated by the observation that the transformer architecture is permutation equivariant: for any permutation \( \pi \in S_n \) and any sequence \( \mathbf X \in \mathbb{R}^{d\times m} \), a transformer \( \tau \) satisfies
\[
\tau(\mathbf X_{\pi(1)}, \ldots, \mathbf X_{\pi(n)}) = \left(\tau(\mathbf X)_{\pi(1)}, \ldots, \tau(\mathbf X)_{\pi(n)}\right).
\]
This symmetry motivates replacing the sequence \( \mathbf X \) with its associated empirical measure
\[
\mathrm{M}(\mathbf X) := \frac{1}{m} \sum_{i=1}^m \delta_{\mathbf X_i},
\]
so that a transformer can be viewed as a transformation of probability measures supported on \( \mathbb{R}^d \), independent of the ordering of tokens. We write $\mathcal G:=\operatorname{Im}(\mathrm{M})=\{\mathrm M(\mathbf X),\mathbf X\in\mathbb R^{d\times m}\}$ the image of $\mathrm{M}$. To extend the action of transformers beyond empirical measures to general probability distributions, we introduce the notion of pushforward. 
\begin{definition}[\citet{santambrogio2015optimal}]
    Given a probability measure $\mu$ on \(\mathbb{R}^d \) and a measurable map \( \varphi: \mathbb{R}^d \to \mathbb{R}^d \), the \emph{pushforward} of \( \mu \) by \( \varphi \), denoted by \( \varphi_{\sharp} \mu \), is the probability   measure defined on Borel sets \( B \subset \mathbb{R}^d \) by
\[
(\varphi_{\sharp} \mu)(B) := \mu(\varphi^{-1}(B)).
\]
\end{definition}

Intuitively, \( \varphi_{\sharp} \mu \) is obtained by transporting mass from each point \( x \) to its image \( \varphi(x) \), preserving total measure. We now define a mean-field transformer.

\begin{definition}[Mean-Field Self-Attention~\citep{Castin_V_2024_p-icml_lipschitz}]
\label{def:meanfield_attention}
The mean-field generalization $F$ of any self-attention layer---parameterized by projection matrices $\mathbf{W}_q$, $\mathbf{W}_k$, $\mathbf{W}_v$, and $\mathbf{W}_o$ as defined in Definition~\ref{def:attention}---is defined by
\[
F: \mu \in \mathcal{P}_{\mathrm c}(\mathbb{R}^d) \mapsto (\Gamma_\mu)_{\sharp} \mu \in \mathcal{P}_{\mathrm c}(\mathbb{R}^d),\quad\mbox{where}
\]
\[
\Gamma_\mu(\mathbf x) := \sum_{i=1}^{h}\frac{\int  \mathbf{W}_{\mathrm o}^i \mathbf{W}_{\mathrm v}^i \mathbf y \cdot \exp\big((\mathbf{W}_{\mathrm k}^i \mathbf{y})^\top \mathbf{W}_{\mathrm q}^i \mathbf{x} \big) \, \mathrm{d} \mu(\mathbf y)}{\int \exp\big((\mathbf{W}_{\mathrm k}^i \mathbf{y})^\top \mathbf{W}_{\mathrm q}^i \mathbf{x} \big) \, \mathrm{d} \mu(\mathbf y)},
\]

Mean-field self-attention $F$ generalizes discrete self-attention $\operatorname{Att}$ in the sense that for any input \(\mathbf X \in \mathbb{R}^{d\times m} \), we have \( F(\mathrm{M}(\mathbf X)) = \mathrm{M}(\operatorname{Att}(\mathbf X,\mathbf X)) \).
\end{definition}

\begin{definition}[Mean-Field Transformer Layer]

Similarly, any transformer layer $\tau$, with MLP layer $\operatorname{MLP}$ and mean-field self-attention layer defined by $\Gamma_\cdot$, has the following mean-field generalization $T$.
\[
T: \mu \in \mathcal{P}_{\mathrm c}(\mathbb{R}^d) \mapsto (\Delta_\mu)_{\sharp} \mu \in \mathcal{P}_{\mathrm c}(\mathbb{R}^d),\quad\mbox{where}
\]
\[
\Delta_\mu(\mathbf x) := \operatorname{MLP}(\Gamma_\mu(\mathbf x)+\mathbf x).
\]

Similarly to mean-field self-attention, we prove in Appendix~\ref{app:mean-field} that a mean-field transformer layer $T$ generalizes a discrete transformer layer $\tau$ in the sense that for any input \(\mathbf X \in \mathbb{R}^{d\times m} \), we have \( T(\mathrm{M}(\mathbf X)) = \mathrm{M}(\tau(\mathbf X)) \).
    
\end{definition}

The study of the mean-field framework requires the introduction of a distance on the set of compactly supported probability measures on \( \mathbb{R}^d \). This motivates the following definition.

\begin{definition}[$q$-Wasserstein Distance~\citep{santambrogio2015optimal}]\label{def:wasserstein}
    For $q\geq1$, the $L_q$ Wasserstein distance on $\mathcal{P}_{\mathrm c}(\mathbb{R}^d)$ is given by
    \[
    W_q(\mu, \nu) := \left( \inf_{\pi \in \Pi(\mu, \nu)} \int \|x - y\|^q \, \mathrm{d}\pi(x, y) \right)^{1/q},
    \]
    for \( \mu, \nu \in \mathcal{P}_{\mathrm c}(\mathbb{R}^d) \), where \( \Pi(\mu, \nu) \) is the set of couplings between \( \mu \) and \( \nu \), {\it i.e.} of probability measures \( \pi \in \mathcal{P}(\mathbb{R}^d \times \mathbb{R}^d) \) such that \( \int \pi(\cdot, y)\, \mathrm{d}y = \mu \) and \( \int \pi(x, \cdot)\, \mathrm{d}x = \nu \).
\end{definition}

\subsection{The Lipschitz Constant of Transformers}

One of the main property of transformers and their mean-field generalizations that we use in this study is that they are Lipschitz.

\begin{definition}[Lipschitz]
    A function $f:E\longrightarrow F$ is said to be $L$-Lipschitz for some $L>0$ with regards to a distance $d$ if $$\forall x,y\in E,d\big(f(x),f(y)\big)\le Ld(x,y).$$
\end{definition}

\begin{remark}
    Recall that to every norm $\|\cdot\|$ is associated the distance $$d(x,y)=\|x-y\|.$$
\end{remark}

\subsubsection*{Unbounded Setting.}
Without bounding the input in a ball of radius \( r \), self-attention (and therefore a transformer) is not globally Lipschitz. In particular, \citet{Kim_H_2021_p-icml_not-lipshitz} show that the Lipschitz constant becomes unbounded as the norm of the inputs diverges. Consequently, all existing upper bounds require restricting the input sequences to a bounded subset of \( \mathbb{R}^d \). This is not a significant limitation in practice, since the input space is finite and therefore necessarily bounded.

\subsubsection*{Lipschitz Properties of Self-Attention.}
The value of the Lipschitz constant for standard single-head self-attention has been rigorously analyzed by~\citet{Castin_V_2024_p-icml_lipschitz}, who establish both upper and lower bounds in the finite-token regime, when considering the Frobenius norm $\|\cdot\|_{\mathrm F}$. The bounds are stated using the operator norm induced by $\|\cdot\|_{\mathrm F}$, which we will denote $\|\cdot\|_{\mathrm{op}}$.

\begin{proposition}[\citet{Castin_V_2024_p-icml_lipschitz}]\label{prp:lip_unmasked_sa}
    When inputs lie in a compact ball of radius \( r \), single-head self-attention with parameters \( (\mathbf A=\mathbf W_{\mathrm k}^\top\mathbf W_{\mathrm q}, \mathbf W_{\mathrm v}) \) is Lipschitz continuous, with constant bounded by
    \[
    \operatorname{Lip}^{\|\cdot\|_{\mathrm F}} \left( f|_{B^n(0,r)} \right)
    \leq \sqrt{3} \|\mathbf W_{\mathrm v}\|_\mathrm{op} \sqrt{ \|\mathbf A\|_\mathrm{op}^2 r^4 (4n + 1) + n}.
    \]
\end{proposition}

This estimate captures the \( \sqrt{n} \) growth of sensitivity with respect to the sequence length. This growth rate is tight for moderate \( n \), specifically when
\begin{align*}
    &n \leq 1+\exp(2r^2\gamma),\quad\mbox{where}\\
    &\gamma=\max(-\gamma_{\mathrm{min}},\frac{\gamma_{\mathrm{max}}}8),
\end{align*}
with $\gamma_{\mathrm{min}}$ and $\gamma_{\mathrm{max}}$ the minimal and maximal real eigenvalues of $\mathbf A$ respectively.

\subsubsection*{Mean-Field Regime.}
As the sequence length \( n \) increases, the bound above becomes loose, eventually diverging as \( n \to \infty \) with fixed radius \( r \). The mean-field regime allows to overcome this limitation.

\begin{proposition}[\citet{NEURIPS2023_b2b3e1d9}]
    When inputs lie in a compact ball of radius \( r \), mean-field single-head self-attention with parameters \( (\mathbf A=\mathbf W_{\mathrm k}^\top\mathbf W_{\mathrm q}, \mathbf W_{\mathrm v}) \) is Lipschitz continuous, with constant bounded by
    \begin{align*}
        \operatorname{Lip}^{W_2}\left( F|_{\mathcal{P}(B^n(0,r))} \right)
    \leq \|\mathbf W_{\mathrm v}\|_{\mathrm{op}} (1& + 3 \|\mathbf A\|_{\mathrm{op}} r^2)\exp\left(2 \|\mathbf A\|_{\mathrm{op}} r^2\right).
    \end{align*}
    
\end{proposition}

\begin{remark}[\citet{Castin_V_2024_p-icml_lipschitz}]
    Let \( r > 0 \). Then,
    \[
    \operatorname{Lip}^{\|\cdot\|_{\mathrm{F}}} \left( f\big|_{B^n(0,r)} \right)
    \leq \operatorname{Lip}^{W_2} \left( F\big|_{\mathcal{P}(B^n(0,r))} \right).
    \]
    This inequality illustrates that the Lipschitz constant of the mean-field self-attention map upper bounds that of its finite-token counterpart, thereby connecting the two frameworks.
\end{remark}

\subsection{Masked Self-Attention}

While most existing theoretical work on transformers---including all prior analyses of prompt tuning cited in this paper---focuses on \emph{unmasked} self-attention, this framework does not capture the architecture of \emph{decoder-only} models~\citep{j.2018generating,openai2024gpt4technicalreport}. These models employ \emph{masked} self-attention, in which each token attends only to its past, inducing a sequential structure. In this work, we incorporate masked self-attention into our analysis.
\begin{definition}
    Given a self-attention operator \( f(\mathbf X)=\operatorname{Att}(\mathbf X,\mathbf X) \) (as in Definition~\ref{def:attention}), we define masked self-attention as the map 
\begin{align*}
     f^{\mathrm m} \colon \bigcup_{m\in\mathbb N}\mathbb{R}^{d\times m} \to \bigcup_{m\in\mathbb N}\mathbb{R}^{d\times m},  \quad\mbox{such that}\\
     f^m(\mathbf X)_i := f([\mathbf X_1, \dots, \mathbf X_i])_i
     \quad \text{for } \mathbf X \in \mathbb R^{d\times m}.
\end{align*}
\end{definition}

Proposition~\ref{prp:lip_unmasked_sa} still holds for masked self-attention.

\begin{proposition}[\citet{Castin_V_2024_p-icml_lipschitz}]
    When inputs lie in a compact ball of radius \( r \), single-head masked self-attention with parameters \( (\mathbf A=\mathbf W_{\mathrm k}^\top\mathbf W_{\mathrm q}, \mathbf W_{\mathrm v}) \) is Lipschitz continuous, with constant bounded by
    \begin{align*}
        \operatorname{Lip}^{\|\cdot\|_{\mathrm F}} \left( f^{\mathrm m}|_{B^n(0,r)} \right)
    \leq \sqrt{3} \|&\mathbf W_{\mathrm v}\|_\mathrm{op}\sqrt{ \|\mathbf A\|_\mathrm{op}^2 r^4 (4n + 1) + n}.
    \end{align*}
\end{proposition}

Since masked self-attention is not permutation invariant, it is not as straightforward to extend $f^{\mathrm m}$ in the mean-field regime. The trick is to extend the input space to \([0,1] \times \mathbb{R}^d\), allowing each token to carry a timestamp representing its location in the sequence. 

\subsubsection*{Position-Aware Distance for Masked Self-Attention}
To study the Lipschitz regularity of masked self-attention, the standard Wasserstein distance is not suitable. Indeed, mean-field masked self-attention operates on inputs of the form \( (s, \mathbf x) \in [0,1] \times \mathbb{R}^d \), where the extra coordinate \( s \) encodes the position in the sequence. The Wasserstein distance, however, allows moving mass between points with different positions \( s \neq s' \), which breaks the sequential structure of masked attention. To fix this, \citet{Castin_V_2024_p-icml_lipschitz} introduce a new distance on \( \mathcal{P}_c([0,1] \times \mathbb{R}^d) \) that only compares points with the same position. With this new distance $d^{\mathrm m}$, the Lipschitz constant of mean-field masked self-attention has the same upper bound as its unmasked counterpart.

\begin{proposition}[\citet{NEURIPS2023_b2b3e1d9}]
    When inputs lie in a compact ball of radius \( r \), mean-field single-head masked self-attention with parameters \( (\mathbf A=\mathbf W_{\mathrm k}^\top\mathbf W_{\mathrm q}, \mathbf W_{\mathrm v}) \) is Lipschitz continuous, with constant bounded by
    \begin{align*}
        \operatorname{Lip}^{d^{\mathrm m}}\left( F^{\mathrm m}|_{\mathcal{P}(B^n(0,r))} \right)
    \leq \|\mathbf W_{\mathrm v}\|_{\mathrm{op}} (1& + 3 \|\mathbf A\|_{\mathrm{op}} r^2)\exp\left(2 \|\mathbf A\|_{\mathrm{op}} r^2\right).
    \end{align*}
    
\end{proposition}

%% file: ch_covering-number.tex
\section{Covering and Packing Numbers}\label{sec:covering}

Our proofs require the notions of covering and packing numbers~\citep{Vershynin_R_2018_b_covering-number}.

\begin{definition}[Covering Number]
Let \( (T, d) \) be a metric space and let \( \varepsilon > 0 \).

The \emph{\( \varepsilon \)-covering number} of \( T \), denoted \( \mathcal{N}(T, d, \varepsilon) \), is the minimal number of balls of radius \( \varepsilon \) (with respect to the metric \( d \)) needed to cover \( T \). That is,
\begin{align*}
    \mathcal{N}(T, d, \varepsilon)=\min\{ N \in \mathbb{N} , \enspace\exists (x_i)_{i\in[N]}\in T^N, \enspace T \subset \bigcup_{i=1}^N B(x_i, \varepsilon) \}.
\end{align*}

\end{definition}

\begin{definition}[Packing Number]
    Let \( (T, d) \) be a metric space and let \( \varepsilon > 0 \).

    The \emph{\( \varepsilon \)-packing number} of \( T \), denoted \( \mathcal{M}(T, d, \varepsilon) \), is the maximal number of disjoint open balls of radius \( \frac\varepsilon2 \) that can be placed in \( T \), or equivalently, the largest cardinality of an \( \varepsilon \)-separated subset of \( T \). That is,
    \begin{align*}
        \mathcal{M}(T, d, \varepsilon)=\max\{ M \in \mathbb{N} ,\enspace\exists (x_i)_{i\in[M]}\in T^M, \enspace\forall i \neq j, d(x_i, x_j) > \varepsilon\}.
    \end{align*}
\end{definition}

Notice that the covering and packing numbers are essentially equivalent:

\begin{lemma}[Approximate Equivalence of Covering and Packing Numbers {\cite[Lemma~4.2.8]{Vershynin_R_2018_b_covering-number}}]
For any set \( K \subset T \) and any \( \varepsilon > 0 \), we have
\[
\mathcal{P}(K, d, 2\varepsilon) \leq \mathcal{N}(K, d, \varepsilon) \leq \mathcal{P}(K, d, \varepsilon).
\]
\end{lemma}

We use the following main result to analyze the setting of bounded length prompts (Section~\ref{sec:prompt_information}).

\begin{proposition}[{\citet[Proposition~4.2.10]{Vershynin_R_2018_b_covering-number}}]
Let $n\in\mathbb N$, \( K \subset \mathbb{R}^n \) and \( \varepsilon > 0 \). Then,
\begin{align*}
    \frac{\mathrm{Vol}(K)}{\mathrm{Vol}(\varepsilon B^n(0,1))} 
\ \leq\ \mathcal{N}(K, \|\cdot\|, \varepsilon),\quad\mbox{and}\\ 
\ \ \mathcal{P}(K, \|\cdot\|, \varepsilon) 
\ \leq\ \frac{\mathrm{Vol}(K + (\varepsilon/2)B^n(0,1))}{\mathrm{Vol}((\varepsilon/2) B^n(0,1))}.
\end{align*}
\end{proposition}

\subsubsection*{The mean-field framework}

We obtain similar results for the study of prompts of arbitrary length in the mean-field framework (Section~\ref{sec:transformer_memory}).

\begin{proposition}[{\citet[Lemma~4.b]{Nguyen}}]
    Let $\mathcal G$ be the set of discrete probability measures on the token embeddings of dimension $d$ as defined in Section~\ref{sec:mean-field}. Then for any $q\ge1,\varepsilon>0$,
    \[\log N(\mathcal{G}, W_q,2\varepsilon) \leq N(\Theta, \|\cdot\|, \varepsilon) \log\Big(e + \frac{e \, \mathrm{Diam}(\Theta)^q}{\varepsilon^q}\Big),\]
    with $\Theta=B^d(0,r)$ and $e=\exp(1)$.
\end{proposition}

\begin{proposition}\label{prp:klo}
    Let $\mathcal G$ be the set of discrete probability measures on the token embeddings of dimension $d$ as defined in Section~\ref{sec:mean-field}. Then for any $q\ge1,\varepsilon>0$, there exists $C>0$ such that
    \[
    \mathcal{N}(\mathcal G, W_q,\varepsilon)\ge\frac1C\exp(\frac1{\varepsilon^d}).
    \]
\end{proposition}

We prove this result in Appendix~\ref{app:covering-number}.


%% file: ch_limited-memory.tex
\section{The Limited Memorization Capability of Transformers}\label{sec:main}

In this section, we formally demonstrate the limitations of transformers for long prompt memorization. We introduce a few useful definitions in Section~\ref{sec:outputs}. In Section~\ref{sec:prompt_information} we prove that the amount of information memorized by a transformer from a prompt scales at most linearly with the prompt length. We then show in Section~\ref{sec:transformer_memory} that, since the amount of information a transformer can memorize through prompt tuning is limited, the first result translates directly to an incapacity of memorizing long prompts.

\subsection{Accessible Outputs}\label{sec:outputs}

It is useful to introduce a few definitions for the statement and proofs of our main results.

\subsubsection*{Accessible Output Sequences}

\begin{definition}[$\varepsilon$-Distinct Vector Sequences]
    Two vector sequences $\mathbf Y=(\mathbf Y^i)_{i\in[k]},\mathbf Z=(\mathbf Z^i)_{i\in[k]}\in\mathbb (R^{d})^k$ are said to be $\varepsilon$-distinct under norm $\|\cdot\|$ if 
    \begin{equation*}
        \text{there exists an }i\in[k]\text{ such that }\|\mathbf Y^i-\mathbf Z^i\|>\varepsilon.
    \end{equation*}
    Several vector sequences are said to be $\varepsilon$-distinct under norm $\|\cdot\|$ if they are pairwise $\varepsilon$-distinct under norm $\|\cdot\|$. We use the notion of output sequence (respectively input sequences) when the vectors considered are the outputs (respectively inputs) vector sequences of a transformer.
\end{definition}

\begin{definition}[$\varepsilon$-Accessible Output Sequence]
    For a fixed transformer $\tau$, a pre-prompt $\mathbf P\in\mathbb R^{d\times m_p}$ is said to approximate an output sequence $\mathbf Y=(\mathbf Y^i)_{i\in[k]}\in\mathbb (R^{d\times m})^k$ under norm $\|\cdot\|$ up to error $\varepsilon$ for an input sequence $\mathbf X=(\mathbf X^i)_{i\in[k]}\in\mathbb (R^{d\times m})^k$ if 
    \begin{equation}\label{eq:acc_output}
        \forall i\in[k],\quad\|\tau  ([\mathbf P, \mathbf X^{i}   ]   )_{:,m_p:} - \mathbf Y^{i}\|\leq\varepsilon.
    \end{equation}
    An output sequence $\mathbf Y$ is said to be $\varepsilon$-accessible under norm $\|\cdot\|$ by an input sequence $\mathbf X$ if there exists a pre-prompt $\mathbf P$ that approximates $\mathbf Y$ for $\mathbf X$ under norm $\|\cdot\|$ up to error $\varepsilon$.
\end{definition}

\begin{remark}\label{rem:role_prompt}
    Historically, theoretical work on prompt tuning removes the first part of the output, corresponding to the pre-prompt, in the approximation objective (Equation~\eqref{eq:acc_output}). That is the goal is to find a pre-prompt $\mathbf P\in\mathbb R^{d\times m_p}$  such that the last part of the output $\tau  ([\mathbf P, \mathbf X^{i}   ]   )_{:,m_p:}$ approximates $\mathbf Y$.
\end{remark}

\subsubsection*{Accessible Output Distributions}

We can adapt the notion of accessible outputs to the mean-field framework.

\begin{definition}[$\varepsilon$-Distinct Vector Distributions]
    Two vector distributions $\mu_Y,\mu_Z\in\mathcal{P}_{\mathrm c}(\mathbb{R}^d)$ are said to be $\varepsilon$-distinct under distance $W_q$ for some $q\ge1$ if 
    \begin{equation*}
       W_q(\mu_Y,\mu_Z)>\varepsilon.
    \end{equation*}
    Several vector distributions are said to be $\varepsilon$-distinct under distance $W_q$ if they are pairwise $\varepsilon$-distinct under distance $W_q$. We use the notion of output distributions when the considered vectors are the output distributions of a transformer.
\end{definition}

\begin{definition}[$\varepsilon$-Accessible Output Distribution]
    For a fixed transformer $\tau$ with mean-field generalization $T$, a pre-prompt $\mathbf P\in\mathbb R^{d\times m_p}$ is said to approximate an output distribution $\mu_Y\in\mathcal{P}_{\mathrm c}(\mathbb{R}^d)$ under distance $W_q$ up to error $\varepsilon$ for an input sequence $\mathbf X\in\mathbb (R^{d\times m})^k$ if $W_q\Big(T  \big(M([\mathbf P, \mathbf X^{i}   ]   )\big) , \mu_Y\Big)\leq\varepsilon$. An output distribution $\mu_Y$ is said to be $\varepsilon$-accessible by an input sequence $\mathbf X$ if there exists $m_p\in\mathbb N$ and a pre-prompt of length $m_p$, $\mathbf P\in\mathbb R^{d\times m_p}$ that approximates $\mu_Y$ under distance $W_q$ for $\mathbf X$ up to error $\varepsilon$.
\end{definition}

\begin{remark}
In contrast to Remark~\ref{rem:role_prompt}, our mean-field formulation does not discard the portion of the transformer's output corresponding to the pre-prompt. Instead, the output is modeled as a full probability distribution over token embeddings, which naturally incorporates the pre-prompt part of the output. Nevertheless, by leveraging positional encodings, one can still enforce constraints specifically on the final part of the output, thereby recovering the standard approximation objective.
\end{remark}

\subsection{The Limit on the Amount of Information that Can Be Contained in A Prompt}\label{sec:prompt_information}

Let us show that the maximal number $k$ of input/output pairs of length $m$ a transformer can reliably learn through prompt tuning of length $m_p$ scales as $k\in O(\frac{m_p}m)$.

\begin{theorem}\label{thm:prompt_information}
    Let $\tau$ be a transformer of Lipschitz constant $L$, embedding radius $r$ and embedding dimension $d$. Then, for $k>m_p\frac{\log(3Lr)-\log(\varepsilon)}{\log(r)-\log(3\varepsilon)} $ and any list of $k$ inputs of size $m$, $\mathbf X=(\mathbf X^1,\dots,\mathbf X^k)\in(\mathbb R^{d\times m})^k$, the proportion (in terms of volume) of output sequences that are $\varepsilon$-accessible is at most $\Big[\frac{(\frac {3Lr}{\varepsilon})^{m_p}}{(\frac {r}{3\varepsilon})^{mk}}\Big]^d\in O((\frac {r}{3\varepsilon})^{-mdk})$.
\end{theorem}

That is the proportion of output sequences that are $\varepsilon$-accessible through prompt tuning decreases exponentially fast with $k\geq C\frac{m_p}m$, where $C$ is a function of the parameters of the transformer and of the target precision $\varepsilon$.

\begin{remark}
    Theorem~\ref{thm:prompt_information} still holds when considering masked self-attention.
\end{remark}

\begin{proof}[Sketch of Proof]
    Our proof can be divided into three steps:
    \begin{itemize}
        \item We obtain a number $C_{\mathrm{out}}(mk)$ of $3\varepsilon$-distinct output sequences.
        \item We prove---by discretizing the pre-prompt space and using the Lipschitz property of $\tau$---that there exists a maximal number $C_{\mathrm{in}}$ of those output sequences that are $\varepsilon$-accessible.
        \item If $C_{\mathrm{out}}(mk)>C_{\mathrm{in}}$, then the proportion of $\varepsilon$-accessible output sequences is at most $\frac{C_{\mathrm{in}}}{C_{\mathrm{out}}}$.
    \end{itemize}

    The formulas for $C_{\mathrm{in}}$ and $C_{\mathrm{out}}$ are obtained using the notions of covering and packing numbers (Section~\ref{sec:covering}).
\end{proof}

The full proof can be found in Appendix~\ref{app:prompt_information}.

\begin{remark}
    Note that Theorem~\ref{thm:prompt_information} requires $r>3\varepsilon$ for the proof to hold. However, this is a natural assumption as the problem becomes trivial if the target precision $\varepsilon$ is of the same order as the maximal norm $r$ of the vectors that are to be approximated.
\end{remark}

\subsection{The Limit on the Amount of Information a Transformer can Memorize through Prompt Tuning}\label{sec:transformer_memory}

Building on the mean-field framework defined in Section~\ref{sec:mean-field}, we prove that the memorization capability of transformers is limited, independantly of prompt size.

\begin{theorem}\label{thm:mean-field}
    Let $\tau$ be a transformer with mean-field generalization $T$ of Lipschitz constant $L$, embedding radius $r$ and embedding dimension $d$. Then for $k>\frac{(\frac{6Lr}\varepsilon)^d(1+\log(1+(\frac{4Lr}\varepsilon)^q))}{(\frac3\varepsilon)^d-\log(C)} $ and any list of $k$ inputs of size $m$, $\mathbf X=(\mathbf X^1,\dots,\mathbf X^k)\in(\mathbb R^{d\times m})^k$, the proportion of output distributions that are $\varepsilon$-accessible is at most $\frac{\Big(e\big(1+(\frac{4Lr}\varepsilon)^q\big)\Big)^{(\frac{6Lr}\varepsilon)^d}}{\big(\frac1C\exp(\frac{3^d}{\varepsilon^d})\big)^{k}}\in O\big(\exp(-k\frac{3^d}{\varepsilon^d})\big)$.
\end{theorem}

That is the proportion of output distributions that are $\varepsilon$-accessible through prompt tuning decreases exponentially fast for large enough $k$.

\begin{remark}
    It is straightforward to extend Theorem~\ref{thm:mean-field} to masked self-attention.
\end{remark}

\begin{proof}[Sketch of Proof]
    The proof follows the same sketch as for Theorem~\ref{thm:prompt_information}, and the full version can be found in Appendix~\ref{app:mean-field_memorization}.
\end{proof}

%% file: ch_single-layer.tex
\section{The Limitations of Prompt Tuning on Single-Layer Transformers}\label{sec:single_layer}

\subsection{Existing Results}

The only current theoretical result on the limitations of the expressivity of prompt tuning in transformers is the following, which relies on a few assumptions.

\begin{assumption}\label{ass:base}
\leavevmode\vspace{0.01em}
\begin{itemize}
        \item $\mathbf W_{\mathrm{q}}, \mathbf W_{\mathrm{k}}, \mathbf W_\mathrm{v}$ are full rank.
        \item $\operatorname{Att}(\mathbf X^i,\mathbf X^i)+\mathbf X^i$ are distinct.
        \item $  (\mathbf Y^{i}   )_{i, k}$ are in the range set of $\operatorname{MLP}$.
    \end{itemize}
\end{assumption}

\begin{assumption}\label{ass:MLP}
    $d \geqslant 2+\operatorname{dim}  [  (\operatorname{MLP}^{-1}  (  (\mathbf y_{10}   )-\mathbf x_{0}   ) \cup  (\operatorname{MLP}^{-1}  (\mathbf y_{20}   )-\mathbf x_{0}   )   ]   $. This condition is satisfied as long as \[
\|\mathbf W_1\|_2 \cdot \|\mathbf W_2\|_2 < 1,
\]
where \( \| \cdot \|_2 \) denotes the matrix spectral norm.

\end{assumption}

\begin{assumption}\label{ass:hidden}
\leavevmode\vspace{0.01em}
    \begin{itemize}
        \item There is only one head: $h=1$.
        \item The symmetric part $\frac{(\mathbf W_{\mathrm{q}}^\top \mathbf W_{\mathrm{k}})+(\mathbf W_{\mathrm{q}}^\top \mathbf W_{\mathrm{k}})^\top}2$ of $\mathbf W_{\mathrm{q}}^\top \mathbf W_{\mathrm{k}}$ has full rank.
    \end{itemize}
\end{assumption}

Notice that Assumption~\ref{ass:base} is especially strong as $\mathbf W_\mathrm{v}$ is typically of rank $s=\frac d h\ll d$ ($s=64, d=768, h=12$ and there are $12$ attention layers for BERT-Base~\citep{Devlin_J_2019_j-acl_BERT}).

\begin{theorem}[\citet{Wang_Y_2023_p-nips_prompt-og}]\label{thm:Wang}
    For a single-layer transformer $\tau$ satisfying Assumptions~\ref{ass:base},~\ref{ass:MLP}, and~\ref{ass:hidden}, there exist $  (\mathbf X^{1}, \mathbf Y^{1}   )$, $  (\mathbf X^{2}, \mathbf Y^{2}   )$ such that $  \forall m_{p} \in \mathbb N, \forall \mathbf P \in \mathbb{R}^{d \times m_{p}}, \tau  ([\mathbf P, \mathbf X^{i}   ]   ) \neq \mathbf Y^{i}$ for some $i\in\{1,2\}$.
\end{theorem}

\subsection{Generalization to weaker assumptions and several heads}

We first show that we can generalize Theorem~\ref{thm:Wang} by keeping only the mild following assumption. 
\begin{assumption}\label{ass:norm_1}
    \[
\|\mathbf W_1\|_2 \cdot \|\mathbf W_2\|_2 < 1,
\]
where \( \| \cdot \|_2 \) denotes the matrix spectral norm.
\end{assumption}

\begin{remark}
    Experimental results in~\citep{dong2021attention} indicate that, for most architectures, the weight matrices have small operator norms. Consequently, the condition \( \|\mathbf W_1\|_2 \cdot \|\mathbf W_2\|_2 < 1 \) is mild and typically satisfied in practice.
\end{remark}

We prove the stronger result that prompt tuning on a single-layer transformer has very little expressiveness in terms of dimensions. That is the transformer cannot memorize most pairs $(\mathbf X^1,\mathbf Y^1),(\mathbf X^2,\mathbf Y^2)$ such that $\mathbf X^1$ and $\mathbf X^2$ share at least one common token.

\begin{theorem}\label{thm:multi_head}
    Let $\tau$ be a 1-layer transformer with $h$ heads such that $d-h^2-2h>0$. Then $\forall \mathbf x_{0}, \mathbf x_{1}, \mathbf x_{2} \in \mathbb{R}^{d}$,
    there exists an $\mathbb{R}$-vector space $E$ of dimension  $\frac{(d-h^2-h)!}{(d-h^2-2h-1)!}$ such that $\forall  (\mathbf y_{1},\dots, \mathbf y_{h+1}   ) \in \operatorname{MLP}(E\setminus\{0\}),$ $\forall m_{p} \in \mathbb N, \forall \mathbf P \in \mathbb{R}^{d\times m_P},   \tau  (  [\mathbf P, \mathbf x_{i}, \mathbf x_{0}   ]   )_{-1} \neq \mathbf y_{i}$ for some $i\in\{1,2\}$. 
\end{theorem}

In other words, "the space accessible through prompt tuning on a single layer is more or less a hyperplane of the space of all available outputs".

\begin{proof}[Sketch of Proof]
    The proof of Theorem~\ref{thm:multi_head} is based on the fact that the single-head attention of $(\mathbf x_0,[\mathbf P,\mathbf x^i,\mathbf x^0])$ can be decomposed as a part depending only on $(\mathbf x^0,\mathbf P)$, and a part depending only of $(\mathbf x^0,\mathbf x^i)$,
    \[
    \operatorname{Att}(\mathbf x_0,[\mathbf P,\mathbf x^i,\mathbf x^0])=\mathbf a_0^{\mathbf P}+\mathbf a_i.
    \]
    So we can see single-head single-layer attention as modifying the last vector of all inputs $\mathbf X^i$ along a same axis $\mathbf a_i$. If the number of independant contraints (in our case the number of outputs $\mathbf y_i$) becomes higher than the number of parameters, then the output becomes unaccessible.
\end{proof}

The proof of Theorem~\ref{thm:multi_head} can be found in Appendix~\ref{app:multi_head}.

\subsection{Generalization to approximate memorization}

We show that prompt tuning on a single-layer transformer won't work even when the goal is relaxed to only learning an approximation of the output $\mathbf Y^i$.

\begin{theorem}\label{thm:multi_head_epsilon}
    Let $\tau$ be a 1-layer transformer with $h$ heads such that $d-h^2-2h>0$. Then $\forall \mathbf x_{0}, \mathbf x_{1}, \mathbf x_{2} \in \mathbb{R}^{d}$,
    there exists an $\mathbb{R}$-vector space $E$ of dimension  $\frac{(d-h^2-h)!}{(d-h^2-2h-1)!}$ such that $\forall  (\mathbf y_{1},\dots, \mathbf y_{h+1}   ) \in \operatorname{MLP}(E\setminus\{0\}),$ $\forall m_{p} \in \mathbb N, \forall \mathbf P \in \mathbb{R}^{d\times m_P},   \tau  (  [\mathbf P, \mathbf x_{i}, \mathbf x_{0}   ]   )_{-1} \geq \frac{(1-\|\mathbf W_1\|_2 \cdot \|\mathbf W_2\|_2)r}{2}$ for some $i\in\{1,2\}$, where  $r=\min_{i\in\{1,2\}}\|\mathbf y_i\|_2$.
\end{theorem}

\begin{proof}[Sketch of Proof]
    Theorem~\ref{thm:multi_head} is stated using an exact matching constraint on each target output vector \( \mathbf{y}_i \). However, the proof remains valid under a relaxed geometric condition: instead of requiring exact recovery, one may demand that the transformer output lies at distance at most \( r/2 \) from each \( \mathbf{y}_i \), where the \( \mathbf{y}_i \) are mutually orthogonal vectors in \( \mathbb{R}^d \) with norm greater than \( r \). Since these vectors form a high-dimensional orthogonal configuration, any point in the space that lies within \( r/2 \) of \( h \) of them must necessarily be at distance at least \( r/2 \) from the remaining one. This relaxation preserves the essential conclusion: the set of simultaneously approximable outputs is confined to a strict subspace of the full output space.
\end{proof}

The full proof can be found in Appendix~\ref{app:multi_head_epsilon}.

%% file: ch_appendices.tex

\section{Proof of the Properties of Mean-Field Generalization}\label{app:mean-field}

Let us prove that mean-field self-attention $F$ generalizes discrete self-attention $\operatorname{Att}$ in the sense that for any input \(\mathbf X \in \mathbb{R}^{d\times m} \), we have \( F(\mathrm{M}(\mathbf X)) = \mathrm{M}(\operatorname{Att}(\mathbf X,\mathbf X)) \).

\begin{proof}
    Let $\mathbf X \in \mathbb{R}^{d\times m}, \mathbf y\in\mathbb R^d$, $\operatorname{Att}$ be the self-attention layer parameterized by projection matrices $\mathbf{W}_q$, $\mathbf{W}_k$, $\mathbf{W}_v$, and $\mathbf{W}_o$ as defined in Definition~\ref{def:attention}, and $F$ be the mean-field generalization of $\operatorname{Att}$. Then
    \begin{align*}
        \mathrm{M}\big(\operatorname{Att}(\mathbf X,\mathbf X)\big)&=\mathrm{M}\big(\left[
        \operatorname{Att}(\mathbf{X}_{:,1}, \mathbf{X}),\;
        \ldots,\;
        \operatorname{Att}(\mathbf{X}_{:,m}, \mathbf{X})
        \right]\big)\\
        &=\mathrm{M}\Big([\mathbf x_l]_{l\in[m]}\Big)\\
        &=\frac{1}{m} \sum_{l=1}^m \delta_{\mathbf x_l},
    \end{align*}
    writing 
    \begin{align*}
        \mathbf x_l&=\sum_{i=1}^{h} \mathbf{W}_{\mathrm o}^i \mathbf{W}_{\mathrm v}^i \mathbf{X} \cdot \sigma\big((\mathbf{W}_{\mathrm k}^i \mathbf{X})^\top \mathbf{W}_{\mathrm q}^i \mathbf{X}_l \big)\\
        &=\sum_{i=1}^{h}\frac{\sum_{j=1}^{m}\mathbf{W}_{\mathrm o}^i \mathbf{W}_{\mathrm v}^i \mathbf{X}_j\exp\big((\mathbf{W}_{\mathrm k}^i \mathbf{X}_j)^\top \mathbf{W}_{\mathrm q}^i \mathbf{X}_l)}{\sum_{j=1}^{m}\exp\big((\mathbf{W}_{\mathrm k}^i \mathbf{X}_j)^\top \mathbf{W}_{\mathrm q}^i \mathbf{X}_l)}\\
        &=\Gamma_{\mathrm{M}(\mathbf X)}(\mathbf X_l).
    \end{align*}
    
    Hence
    \begin{align*}
        \Big(\mathrm{M}(\operatorname{Att}(\mathbf X,\mathbf X))\Big)(\mathbf y)&=\frac1m\sum_{l=1}^m \mathbf{1}_{\mathbf x_l=\mathbf y}\\
        &=\frac1m\sum_{l=1}^m \mathbf{1}_{\Gamma_{\mathrm{M}(\mathbf X)}(\mathbf X_l)=\mathbf y}\\
        &=\frac1m\sum_{l=1}^m \mathbf{1}_{\mathbf X_l\in\Gamma_{\mathrm{M}(\mathbf X)}^{-1}(\mathbf y)}\\
        &=\mathrm{M}(\mathbf X)\big(\Gamma_{\mathrm{M}(\mathbf X)}^{-1}(\mathbf y)\big)\\
        &=\Big((\Gamma_{\mathrm{M}(\mathbf X)})_{\sharp} {\mathrm{M}(\mathbf X)}\Big)(\mathbf y)\\
        &=\Big(F\big(\mathrm{M}(\mathbf X)\big)\Big)(\mathbf y).
    \end{align*}
\end{proof}

We use the same proof to show that a mean-field transformer layer $T$ generalizes a discrete transformer layer $\tau$ in the sense that for any input \(\mathbf X \in \mathbb{R}^{d\times m} \), we have \( T(\mathrm{M}(\mathbf X)) = \mathrm{M}(\tau(\mathbf X)) \).

\begin{proof}
    Let $\mathbf X \in \mathbb{R}^{d\times m}, \mathbf y\in\mathbb R^d$, $\tau$ be a transformer layer with self-attention $\operatorname{Att}$ as described above and MLP layer $\operatorname{MLP}$, and $F$ be the mean-field generalization of $\operatorname{Att}$. Then
    \begin{align*}
        \mathrm{M}(\tau(\mathbf X))&=\mathrm{M}\Big(\operatorname{MLP}\big([
        \operatorname{Att}(\mathbf{X}_{:,l}, \mathbf{X})]_{l\in[m]}
        \big)+\mathbf X\Big)\\
        &=\mathrm{M}\Big([\mathbf x_l]_{l\in[m]}\Big)\\
        &=\frac{1}{m} \sum_{l=1}^m \delta_{\mathbf x_l},
    \end{align*}
    writing 
    \begin{align*}
        \mathbf x_l&=\operatorname{MLP}\Big(\sum_{i=1}^{h} \mathbf{W}_{\mathrm o}^i \mathbf{W}_{\mathrm v}^i \mathbf{X} \cdot \sigma\big((\mathbf{W}_{\mathrm k}^i \mathbf{X})^\top \mathbf{W}_{\mathrm q}^i \mathbf{X}_l \big)\Big)+\mathbf X_l\\
        &=\operatorname{MLP}\Big(\sum_{i=1}^{h}\frac{\sum_{j=1}^{m}\mathbf{W}_{\mathrm o}^i \mathbf{W}_{\mathrm v}^i \mathbf{X}_j\exp\big((\mathbf{W}_{\mathrm k}^i \mathbf{X}_j)^\top \mathbf{W}_{\mathrm q}^i \mathbf{X}_l)}{\sum_{j=1}^{m}\exp\big((\mathbf{W}_{\mathrm k}^i \mathbf{X}_j)^\top \mathbf{W}_{\mathrm q}^i \mathbf{X}_l)}\Big)+\mathbf X_l\\
        &=\Delta_{\mathrm{M}(\mathbf X)}(\mathbf X_l).
    \end{align*}
    
    Hence
    \begin{align*}
        \Big(\mathrm{M}(\tau(\mathbf X))\Big)(\mathbf y)&=\frac1m\sum_{l=1}^m \mathbf{1}_{\mathbf x_l=\mathbf y}\\
        &=\frac1m\sum_{l=1}^m \mathbf{1}_{\Delta_{\mathrm{M}(\mathbf X)}(\mathbf X_l)=\mathbf y}\\
        &=\frac1m\sum_{l=1}^m \mathbf{1}_{\mathbf X_l\in\Delta_{\mathrm{M}(\mathbf X)}^{-1}(\mathbf y)}\\
        &=\mathrm{M}(\mathbf X)\big(\Delta_{\mathrm{M}(\mathbf X)}^{-1}(\mathbf y)\big)\\
        &=\Big((\Delta_{\mathrm{M}(\mathbf X)})_{\sharp} {\mathrm{M}(\mathbf X)}\Big)(\mathbf y)\\
        &=\Big(T\big(\mathrm{M}(\mathbf X)\big)\Big)(\mathbf y).
    \end{align*}
\end{proof}


\section{Proof of the Lower Bounds for the Covering Number}\label{app:covering-number}

We use the critical parameters introduced in~\citet{crit} to prove Proposition~\ref{prp:klo}.

\begin{lemma}
    Let $(\mathcal X,d)$ be a Polish space. If $\operatorname{crit}_{\mathcal P}(\mathcal X)>s$, then there exists a constant $C>0$ such that $\mathcal N(\mathcal X,d,\varepsilon)\ge\frac1C\exp(\frac1{\varepsilon^s})$.
\end{lemma}

\begin{proof}
    From Frostman’s Lemma~\citep{Kloeckner_2014}, we obtain a Borel probability measure $\mu$ on $\mathcal X$ and a constant $C>0$ such that
    \[\forall \mathbf x\in \mathcal X,\forall r>0,\mu(B(\mathbf x,r))\leq C\exp(-\frac1{r^s }).\]
    Then, taking an $\varepsilon$-cover of $\mathcal X$ 
\[\mathcal X \subset \bigcup_{i=1}^{\mathcal{N}(\mathcal X, d, \varepsilon)} B(x_i, \varepsilon),\] 
we get
\begin{align*}
    1&=\mu(\mathcal X)\\
    &\le \sum_{i=1}^{\mathcal{N}(\mathcal X, d, \varepsilon)}\mu(B(x_i, \varepsilon))\\
    &\le\mathcal{N}(\mathcal X, d, \varepsilon)C\exp(-\frac1{\varepsilon^s}).
\end{align*}

Hence
\[
    \mathcal{N}(\mathcal X, d, \varepsilon)\ge\frac1C\exp(\frac1{\varepsilon^s}).
\]

\end{proof}

\begin{proof}[Proof of Proposition~\ref{prp:klo}]
    Proposition~\ref{prp:klo} comes from the fact that $(\mathcal G,W_q)$ is a Polish space~\cite[Chapter 6]{villani2008optimal} satisfying $\operatorname{crit}_{\mathcal P}(\mathcal G)\ge d$~\cite[Theorem~1.3]{Kloeckner_2014}.
\end{proof}


\section{Proof of Theorem~\ref{thm:prompt_information}}\label{app:prompt_information}

\begin{proof}
    Let $k\in\mathbb N$ and $\mathbf X^1,\dots,\mathbf X^k\in\mathbb R^{d\times m}$ be a list of $k$ prompts of length $m$. 
    We can divide the output space $B^{dm}(0,r)$ in $(\frac {r}{3\varepsilon})^{dm}$ balls of radius $\varepsilon$, each distant of at least $\varepsilon$~\citep{Vershynin_R_2018_b_covering-number}. This division yields $C_{\mathrm{out}}:=(\frac {r}{3\varepsilon})^{dmk}$ $3\varepsilon$-distinct output sequences.

    Notice that if $\|\mathbf P-\mathbf P'\|\leq\frac\varepsilon L$, then $\forall\mathbf X\in\mathbb R^{d\times m},\|\tau  ([\mathbf P, \mathbf X   ]   )-\tau  ([\mathbf P', \mathbf X   ]   )\|\leq\varepsilon$. Since $B^{dm_p}(0,r)$ can be covered by $C_{\mathrm{in}}:=(\frac {3Lr}{\varepsilon})^{dm_p}$ balls of radius $\frac\varepsilon L$~\citep{Vershynin_R_2018_b_covering-number}, there are at most $C_{\mathrm{in}}$ of the $\varepsilon$-distinct output sequences that are $\varepsilon$-accessible by $\mathbf X$.
    
    Therefore, for $mk>m_p\frac{\log(3Lr)-\log(\varepsilon)}{\log(r)-\log(3\varepsilon)} $, some output sequences aren't $\varepsilon$-accessible by $\mathbf X$, and the proportion of output sequences that are $\varepsilon$-accessible is at most $\frac{C_{\mathrm{in}}}{C_{\mathrm{out}}}$.
\end{proof}


\section{Proof of Theorem~\ref{thm:mean-field}}\label{app:mean-field_memorization}

\begin{proof}
    Let $k\in\mathbb N$ and $\mathbf X^1,\dots,\mathbf X^k\in\mathbb R^{d\times m}$ be a list of $k$ prompts of length $m$. 
    We can divide the output space $\mathcal{P}_{\mathrm c}(B^{d}(0,r))$ in $\frac1C\exp(\frac{3^d}{\varepsilon^d})$ balls of radius $\varepsilon$, each distant of at least $\varepsilon$ (Section~\ref{sec:covering}). This division yields $C_{\mathrm{out}}:=\big(\frac1C\exp(\frac{3^d}{\varepsilon^d})\big)^{k}$ $3\varepsilon$-distinct output distributions.

    Notice that if $W_q\big(\mathrm M(\mathbf P),\mathrm M(\mathbf P')\big)\leq\frac\varepsilon L$, then $\forall\mathbf X\in\mathbb R^{d\times m},W_q\Big(\mathrm M\big(\tau  ([\mathbf P, \mathbf X   ]   )\big),\mathrm M\big(\tau  ([\mathbf P', \mathbf X   ]   )\big)\Big)\leq\varepsilon$. Since $\mathcal G$ can be covered by $C_{\mathrm{in}}:=\Big(e\big(1+(\frac{4Lr}\varepsilon)^q\big)\Big)^{(\frac{6Lr}\varepsilon)^d}$ balls of radius $\frac\varepsilon L$ (Section~\ref{sec:covering}), there are at most $C_{\mathrm{in}}$ of the $\varepsilon$-distinct output distributions that are $\varepsilon$-accessible by $\mathbf X$.
    
    Therefore, for $k>\frac{(\frac{6Lr}\varepsilon)^d(1+\log(1+(\frac{4Lr}\varepsilon)^q))}{(\frac3\varepsilon)^d-\log(C)} $, some output distributions aren't $\varepsilon$-accessible by $\mathbf X$, and the proportion of output distributions that are $\varepsilon$-accessible is at most $\frac{C_{\mathrm{in}}}{C_{\mathrm{out}}}$.
\end{proof}


\section{Proof of Theorem~\ref{thm:multi_head}}\label{app:multi_head}

\begin{lemma}\label{lem:multi_head}
    Let $\mathbf x_{0}, ..., \mathbf x_{h+1} \in \mathbb{R}^{d}$.
    Write for $i\in[h+1]$,
    $$
    \begin{aligned}
        & \mathbf X^{i}=  [\mathbf x_{i},\mathbf x_{0}   ],\\
    & \mathbf a_{i}^k=\operatorname{Att}^k  (\mathbf x_{0}, \mathbf X^{i}   ), \\
    & \mathbf a_{i}^{\mathbf P}=\operatorname{Att}  (\mathbf x_{0},  [\mathbf P, \mathbf X^{i}   ]   ), \\
    & (\mathbf a_{0}^{\mathbf P})^k=\operatorname{Att}^k  (\mathbf x_{0}, \mathbf P   ).
    \end{aligned}
    $$
    
    That is $\mathbf a_{i}^{\mathbf P}=\sum_{k=1}^{h} \lambda_{i}^{k} \mathbf a_{i}^{k}+\mu_{i}^{k}  (\mathbf a_{0}^{\mathbf P}   )^{k}$ with $\lambda_i^k\in(0;1)$ and $\mu^k_i=1-\lambda_i^k$.
    Write $E=\operatorname{Vect}  (\mathbf a_{i}^{k}   )_{i\in[h+1],k \in[h]}$.
    Then, for all $\mathbf y_{1},..., \mathbf y_{h+1} \in E^{\perp}\setminus\{0\}$ such that $\forall i,j\in[h],\mathbf y_{i} \perp \mathbf y_{j}$, there is no $\mathbf P$ such that $\mathbf a_{i}^{\mathbf P}=\mathbf y_{i}$ for all $i\in[h]$.
\end{lemma}

\begin{proof}
    Assume $\mathbf y_{i}=\sum_{k=1}^{h} \lambda_{i}^{k} \mathbf a_{i}^{k}+\mu_{i}^{k}  (\mathbf a_{0}^{p}   )^{k}$.\\
    Then $(\mathbf a_0^{\mathbf P})^i=\frac1{\mu_i}(\mathbf y_i-\sum_{k=1}^{h} \lambda_{i}^{k}\mathbf a_{i}^{k}-\sum_{k\in[h]\setminus\{i\}}\mu_i^k(\mathbf a_0^{\mathbf P})^k)$\\
    So for all $i\in[h]$, there exists $f_i$ a non-zero linear combination (meaning a linear combination with non-zero derivative on the $\mathbf y_i$) of $\mathbf y_i$ and $(\mathbf a_i^k)_{k\in[h]}$ such that $(\mathbf a_0^{\mathbf P})^i=f_i-\frac1{\mu_i}\sum_{k\in[h]\setminus\{i\}}\mu_i^k(\mathbf a_0^{\mathbf P})^k$.\\
    By induction, for all $i\in[h]$, there exists $g_i$ a non-zero linear combination (meaning a linear combination with non-zero derivative on the $\mathbf y_i$) of $(\mathbf y_j,\mathbf a_j^k)_{j\in [i],k\in[h]}$ and non-zero scalars $\psi_i^k:=\psi_i^k(\mu_a^b,\mu_c^k)_{a,b,c\in[i]}$ such that $(\mathbf a_0^{\mathbf P})^i=g_i((\mathbf y_j,\mathbf a_j^k)_{j\in [i]})+\sum_{k=i+1}^h\psi_i^k(\mathbf a_0^{\mathbf P})^k$.\\
    Therefore, $(\mathbf a_0^{\mathbf P})^h=g_h((\mathbf y_j,\mathbf a_j^k)_{j\in [h]})=\tilde g_h((\mathbf y_j,\mathbf a_j^k)_{j\in [h-1]\cup\{h+1\}})$.\\
    We can conclude by taking the scalar product with regards to $\mathbf y_h$.
\end{proof}

\begin{proof}[Proof of Theorem~\ref{thm:multi_head}]
    Take $\mathbf y_1',\mathbf y_2'$ from Lemma~\ref{lem:multi_head}.
    Write $\mathbf y_{i}=\operatorname{MLP}  (\mathbf y_{i}^{\prime}+\mathbf x_{0}   )$.
\end{proof}


\section{Proof of Theorem~\ref{thm:multi_head_epsilon}}\label{app:multi_head_epsilon}

We provide the proof for $h=1$ for clarity. First, let us generalize Lemma~\ref{lem:multi_head} to the outcome "there is no $\mathbf P$ such that $|\mathbf a_{i}^{\mathbf P}-\mathbf y_{i}|<\frac r{2}$ for all $i\in\{1,2\}$".

\begin{lemma}\label{lem:no_ass_epsilon}
    Let $\mathbf x_{0}, \mathbf x_{1}, \mathbf x_{2} \in \mathbb{R}^{d}$.
    Write $\mathbf a_{i}=\operatorname{Att}  (\mathbf x_{0}, \mathbf X^{i}   )$ and $E=\operatorname{Vect}  (\mathbf a_{i}   )_{i\in\{1,2\}}$.
    Then for all $\mathbf y_{1}, \mathbf y_{2} \in E^{\perp}\setminus\{0\}$ such that $\mathbf y_{1} \perp \mathbf y_{2}$, there is no $\mathbf P$ such that $|\mathbf a_{i}^{\mathbf P}-\mathbf y_{i}|<\frac r{2}$ for all $i\in\{1,2\}$ (where $r=\min_{i\in\{1,2\}}\|\mathbf y_i\|_2$).
\end{lemma}

\begin{proof}
    Assume that $\|\mathbf y_i-\lambda_i\mathbf a_i-\mu_i\mathbf a_0^{\mathbf P}\|_2\leq\varepsilon$ for $i \in \{1,2\}$.

    We thus have $\mathbf a_0^{\mathbf P} = \frac 1{\mu_1}(\mathbf y_1-\lambda_1\mathbf a_1) + \boldsymbol \delta_1 = \frac 1{\mu_2}(\mathbf y_2-\lambda_2\mathbf a_2) + \boldsymbol \delta_2 $ with $\| \boldsymbol\delta_i\| \leq \frac{\varepsilon}{\mu_i}$.
    
    So $\langle\boldsymbol\delta_1,\mathbf y_2\rangle = \frac1{\mu_2}\|\mathbf y_2\|_2^2 + \langle\boldsymbol\delta_2,\mathbf y_2\rangle$
    
    Then
    \begin{align*}
        \|\mathbf y_2\|^2 & = \mu_2\langle\boldsymbol\delta_1 - \boldsymbol\delta_2,\mathbf y_2\rangle\\
        & \leq \mu_2(\frac{\varepsilon}{\mu_1}+\frac{\varepsilon}{\mu_2})\|\mathbf y_2\| \\
        & =\varepsilon(1+\frac{\mu_2}{\mu_1})\|\mathbf y_2\|
    \end{align*}
    Without loss of generality, we assume $\mu_2 \geq \mu_1$. We thus have $\|\mathbf y_2\| \leq 2\varepsilon$. Therefore  $\varepsilon \geq r/2$, which proves the lemma.
\end{proof}

\begin{proof}
    From \citet{Behrmann_J_2019_p-icml_invertible-MLP}, the $\operatorname{MLP}$ is invertible and its inverse has Lipshitz constant $\frac1{1-\|W_1\|_2 \cdot \|W_2\|_2}$.
    Take $\mathbf y_1',\mathbf y_2'$ from Lemma~\ref{lem:multi_head}.
    Write $\mathbf y_{i}=\operatorname{MLP}  (\mathbf y_{i}^{\prime}+\mathbf x_{0}   )$.
\end{proof}